\pdfoutput=1
\documentclass[11pt]{article}

\usepackage[margin=1in]{geometry}

\usepackage[utf8]{inputenc}
\usepackage[T1]{fontenc}
\usepackage{microtype}

\usepackage{amsmath}
\usepackage{amssymb}
\usepackage{amsfonts}
\usepackage{mathtools}
\usepackage{amsthm}
\usepackage{authblk}

\usepackage{graphicx}
\usepackage{booktabs}
\usepackage{multirow}
\usepackage{pifont}
\usepackage{makecell}
\usepackage[table]{xcolor}
\usepackage{wrapfig}
\usepackage{enumitem}
\usepackage{subcaption}
\usepackage{algorithm}
\usepackage{algpseudocode}
\usepackage{nicefrac}
\usepackage{url}
\usepackage[numbers,sort&compress]{natbib}
\usepackage[hidelinks]{hyperref}
\usepackage[capitalize,noabbrev]{cleveref}

\newcommand{\cmark}{\ding{51}}
\newcommand{\xmark}{\ding{55}}
\newcommand{\sysname}{MIITA}

\theoremstyle{plain}
\newtheorem{theorem}{Theorem}[section]
\newtheorem{proposition}[theorem]{Proposition}

\theoremstyle{definition}

\theoremstyle{remark}

\newtheorem*{informalthm}{Theorem (Directional coverage bound, informal)}

\title{\sysname{}: Memory-Induced Inference-Time Adaptation for Continual Learning with Small Language Models}

\author[1]{Dong Li}
\author[2]{Yanchi Liu}
\author[2]{Xujiang Zhao}
\author[2]{Wei Cheng}
\author[2]{Zhengzhang Chen}
\author[3]{Xintao Wu}
\author[4]{Zhong Chen}
\author[1]{Chen Zhao}
\author[2]{Haifeng Chen}

\affil[1]{
Baylor University\\
\texttt{\{dong\_li1, chen\_zhao\}@baylor.edu}
}

\affil[2]{
NEC Laboratories America\\
\texttt{\{yanchi, xuzhao, weicheng, zchen, haifeng\}@nec-labs.com}
}

\affil[3]{
University of Arkansas\\
\texttt{xintaowu@uark.edu}
}

\affil[4]{
Southern Illinois University\\
\texttt{zhong.chen@cs.siu.edu}
}

\date{}

\begin{document}

\maketitle

\begin{abstract}
Continual learning (CL) is essential for small language models (SLMs) to adapt to evolving real-world needs in resource-constrained deployments. However, directly updating their limited parameter space causes catastrophic forgetting. While memory-based methods naturally address this by decoupling knowledge retention from parameters, existing approaches designed for large language models (LLMs) rely on abundant storage and strong in-context reasoning that SLMs lack. To address these challenges, we propose \sysname{}, a Memory-Induced Inference-Time Adaptation framework for supervised CL under constrained storage. \sysname{} stores supervised experiences as compact correction-direction prototypes with semantic anchors, and retrieves them at inference time using semantic and uncertainty-based cues. 
The retrieved directions are applied through gated temporary hidden-state adaptation, enabling non-destructive reuse of past supervision without backbone updates, prompt extensions, or test-time backpropagation. A local theoretical analysis links this design to first-order loss reduction, uncertainty-guided retrieval, and directional coverage for retaining old-stage knowledge. 
Extensive experiments across diverse supervised CL settings show that \sysname{} consistently improves final performance and mitigates forgetting under fixed memory budgets.
\end{abstract}

\section{Introduction}
\label{sec:introduction}

Small language models (SLMs) are increasingly being deployed in localized, low-latency, privacy-sensitive, and resource-constrained settings, where they need to continually adapt to evolving real-world needs~\cite{lu2024small,van2025survey}. However, continual learning (CL) for SLMs is fundamentally difficult to achieve through the parameter space alone. Their limited capacity and high representational redundancy mean that repeatedly writing new experiences into the backbone induces severe catastrophic forgetting and unstable updates~\cite{DKVB,DA-GRPO,POS}. 
Therefore, the key challenge is to continuously utilize new supervised experiences without overwriting the scarce parameter space.

Memory-based CL naturally addresses this by decoupling experience retention from the model's parameter space, allowing the backbone to remain stable while auxiliary storage accumulates new knowledge~\cite{shi2025continual,mok2023large,gutierrez2024hipporag}. However, existing LLM-centric memory methods translate poorly to SLMs because they typically preserve past experience as textual examples, summaries, or demonstrations~\cite{CIS,das2024larimar}. Such memories are storage-inefficient under a fixed budget: they retain surface content in addition to the reusable correction signal, and functionally redundant experiences are difficult to consolidate in text form~\cite{luo2025empirical}. Moreover, they rely on strong in-context reasoning to infer how retrieved memories should affect predictions, which is constrained in smaller models~\cite{wang2025comprehensive}. As shown in Figure~\ref{fig:intro}, the performance of representative memory-based LLM CL methods degrades significantly as the base model size decreases, especially under fixed storage budgets. This trend highlights a critical dual bottleneck for SLMs: the inability to maintain massive textual memories under tight storage limits, and the limited capacity to effectively exploit them during inference.

\begin{wrapfigure}{R}{0.6\textwidth}
    \centering
    \setlength{\abovecaptionskip}{0pt}
    \setlength{\belowcaptionskip}{-2pt}
    \includegraphics[width=\linewidth]{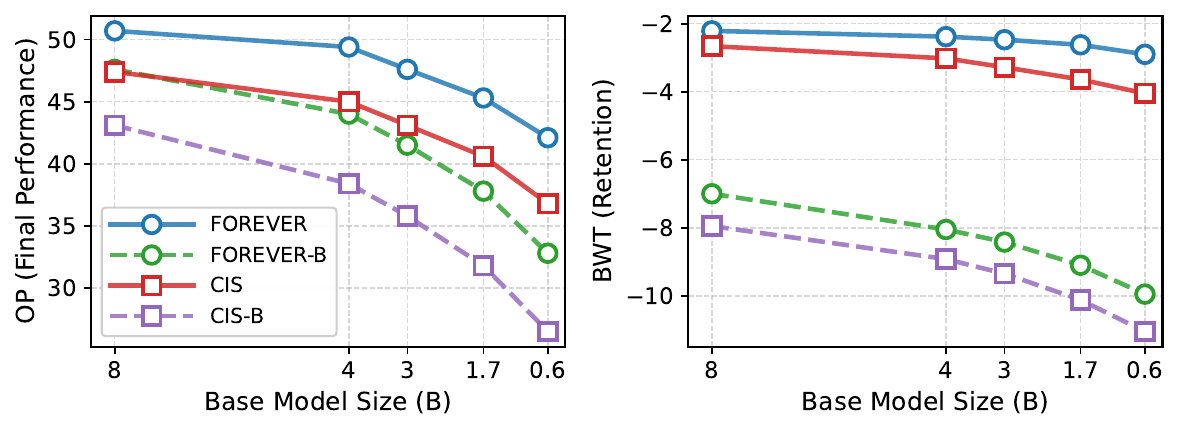}
    \caption{
Model-size scaling of CIS~\cite{CIS} and FOREVER~\cite{FOREVER}, two LLM-oriented memory-based CL baselines, on \texttt{SuperNI}~\cite{SuperNI} benchmark. The ``-B'' suffix indicates variants restricted to a fixed memory budget, contrasting with their unconstrained counterparts. We evaluate Overall Performance (OP) and Backward Transfer (BWT); higher values indicate better accuracy and less forgetting, respectively.
}
    \label{fig:intro}
\end{wrapfigure}
Overcoming this dual bottleneck requires rethinking what a memory item should preserve. Inspired by von Oswald \textit{et al.}~\cite{ICL-as-Opt}, who characterize in-context examples as inducing gradient-descent-like local corrections during the forward pass, the value of a memory item lies not in its raw semantic content alone, but in the functional correction it imposes on the model's hidden representation. This perspective reveals a \textit{divergence between semantic and functional redundancy}: semantically disparate examples may induce identical corrections and can thus be merged, whereas semantically similar examples might require distinct corrections. Consequently, rather than preserving historical text and expecting the SLM to infer how to use it, an alternative paradigm is to organize memory directly around reusable adaptation directions. Storing these functional signals in a compact form addresses both tight storage limits and the reliance on in-context reasoning.

Motivated by this paradigm, we propose \textbf{\sysname{}}, a \textbf{Memory-Induced Inference-Time Adaptation} framework for supervised CL with SLMs under constrained persistent storage. \sysname{} converts each supervised experience into a compact correction direction within the model's internal hidden space, consolidating these functional signals into budget-aware memory prototypes that utilize semantic anchors purely as lightweight retrieval keys. At inference time, the framework retrieves relevant historical directions through both semantic relevance and an uncertainty-based directional proxy. These directions are aggregated into a query-specific correction and applied via a gated, temporary hidden-state update to directly guide the prediction. This adaptation is disposable, requiring no permanent parameter updates, test-time backpropagation, or lengthy prompt extensions.
Our main contributions are summarized as follows:
\begin{itemize}[leftmargin=*, itemsep=0pt, topsep=0pt]
    \item We propose a direction-oriented memory representation that stores supervised experiences as correction-direction prototypes with lightweight semantic anchors, overcoming the high storage and reasoning costs of traditional textual memory in SLMs.

    \item We introduce a memory-induced inference-time adaptation mechanism that applies retrieved historical directions via gated temporary hidden-state updates, safely reusing past supervision without test-time backpropagation or destructive backbone updates.

    \item We provide a local theoretical analysis showing that correction directions are first-order optimal for decreasing the supervised loss, that uncertainty directions identify the steepest direction for decreasing local predictive entropy and serve as label-free retrieval cues, and that old-stage retention is controlled by the directional coverage of stored memory.

    \item We provide comprehensive empirical evaluation across diverse CL settings, showing that \sysname{} improves final performance and reduces forgetting under fixed memory budgets, with further analyses demonstrating its robustness across model scales and memory budgets, as well as the effectiveness of functional correction memory and uncertainty-guided retrieval..
\end{itemize}

\section{Related Work}
\label{sec:related}
\textbf{CL for SLMs.} CL for SLMs is crucial for resource-constrained deployment but remains underexplored. Existing methods address catastrophic forgetting through various mechanisms: DKVB~\cite{DKVB} uses localized updates via a discrete key-value bottleneck; DA-GRPO~\cite{DA-GRPO} optimizes continual post-training while learning to invoke a cloud LLM under a cloud-assistance budget; and Srinath K \textit{et al.}~\cite{POS} use POS-guided code-switching to mitigate cross-lingual forgetting through replay adapters. However, these approaches still rely on training-time parameter updates, cloud assistance, or sample replay. 
In contrast, \sysname{} addresses these limitations by bypassing persistent backbone modifications and training-time replay, thereby overcoming the inherent unreliability of SLMs in continual fine-tuning and text-based prompting.

\textbf{Memory-Based CL for LLMs.} Existing memory-based CL methods for LLMs generally fall into two categories based on when memory is used. Training-time replay methods reuse stored examples during optimization; for example, Continual-T0 (CT0)~\cite{CT0} employs a small rehearsal buffer to retain previous abilities, while FOREVER~\cite{FOREVER} schedules memory reviews based on update dynamics. Alternatively, inference-time methods like CIS~\cite{CIS} and InCA~\cite{InCA} utilize textual summaries or in-context examples to avoid forgetting without updating the LLM. Other inference-time methods (e.g., MAC~\cite{MAC}, CMT~\cite{CMT}, HippoRAG~2~\cite{HippoRAG2}, MBC~\cite{MBC}) retrieve external documents for continual knowledge updates. However, these methods primarily focus on knowledge augmentation rather than directly addressing task-level forgetting in supervised CL, which differs from our target problem. Furthermore, the remaining LLM-centric approaches are ill-suited for SLMs, where limited capacity and weaker in-context abilities make replay and textual memories unreliable. In contrast, \sysname{} stores supervised experiences as compact functional correction directions, invoking them via lightweight, optimization-free hidden-state adaptation at inference time.

\section{Methodology}
\label{sec:method}
\subsection{Problem Setting}

We study \textbf{supervised continual learning (CL) for small language models (SLMs) under a constrained persistent auxiliary storage budget}. 
An SLM receives a sequence of data stages 
$\mathcal{D}_{\mathrm{stream}}=\{\mathcal{D}_1,\ldots,\mathcal{D}_T\}$, where each stage $\mathcal{D}_t=\{(x_i^t,y_i^t)\}_{i=1}^{N_t}$ contains supervised input-output examples. 
After each stage, a CL algorithm updates its system state, which may include model parameters and/or auxiliary persistent information such as replay samples, summaries, prompts, adapters, statistics, or memory items. 
To reflect the resource constraints of SLM deployment, the total size of this auxiliary persistent information is bounded by a fixed storage budget $C_{\mathrm{mem}}$. 
The goal is to learn from each new stage while preserving performance on all previously seen stages, i.e., maximizing final average performance and minimizing forgetting.

\subsection{\sysname{}: Memory-Induced Inference-Time Adaptation}
\label{sec:miita_overview}

To address budget-constrained supervised CL for SLMs, we introduce \textbf{\sysname{}}, a memory-induced inference-time adaptation framework that decouples long-term experience retention from the frozen model parameter space by preserving historical supervision as compact correction directions rather than raw examples or textual summaries.
Tailored to SLMs' limited capacity and weaker in-context abilities,
\sysname{} models continual memory as reusable adaptation signals: during CL, supervised experiences are converted into pre-LM-head correction directions, which the memory manager then organizes into budget-aware prototypes with semantic keys for retrieval. 
At inference, \sysname{} retrieves relevant prototypes using semantic and uncertainty cues, applying them via a gated temporary hidden update.
We provide local theoretical analyses for correction-direction extraction, uncertainty-based retrieval, temporary hidden adaptation, and direction-memory coverage; formal statements and proofs are given in Appendix~\ref{app:theory}.
The overall framework of \sysname{} is shown in Algorithm~\ref{alg:miita} in Appendix~\ref{app:miita_algorithm}


\subsubsection{Direction-Oriented Memory Representation}
\label{sec:memory_unit}

\sysname{} maintains a memory bank $\mathcal B=\{u_1,\ldots,u_J\}$, where each unit is a direction prototype:
\[
u_j=(d_j,\mathcal A_j,n_j).
\]
This design reflects the need to separate functional retention from semantic access in SLMs, which cannot reliably depend on large textual memories due to tighter context budgets and weaker in-context abilities. 
The prototype direction $d_j\in\mathbb R^{H}$ is a normalized correction direction in the pre-LM-head hidden space and serves as the primary stored signal, preserving the functional adaptation effect of historical supervised experiences. 
The anchor set $\mathcal A_j$ contains lightweight semantic keys that are used only to retrieve this functional signal from different query contexts, and $n_j$ records how many experiences have been merged into the prototype. 
The memory bank $\mathcal B$ is maintained under the global storage budget $C_{\mathrm{mem}}$, which is enforced during memory writing through memory repair when new directions or anchors exceed the budget, as detailed in Section~\ref{sec:memory_writing}.

\subsubsection{Extracting Direction Signal from an Experience}
\label{sec:direction_extraction}

Given a supervised experience $(x_m,y_m)$, \sysname{} extracts two signals: a semantic key $k_m$ for memory access and a correction direction $d_m$ for memory writing.

\textbf{Semantic key.}
Let $\mathcal X_m$ be the input token positions and $h_r^{(\ell_k)}$ the hidden state at layer $\ell_k$. We define
\begin{equation}
k_m=
\mathrm{Norm}
\left(
\frac{1}{|\mathcal X_m|}
\sum_{r\in\mathcal X_m}
h_r^{(\ell_k)}
\right).
\label{eq:semantic_key}
\end{equation}
where $\mathrm{Norm}(v)=v/(\|v\|_2+\epsilon)$. The key $k_m$ is used only for later retrieval.

\textbf{Correction vector.}
We define the correction signal in the pre-LM-head hidden space, i.e., the space of final hidden states before they are projected to vocabulary logits. 
For an experience $(x_m,y_m)$, its unnormalized correction vector $R_m$ is
\begin{equation}
R_m=
\frac{1}{|\mathcal Y_m|}
\sum_{t\in\mathcal Y_m}
\left(
-\frac{\partial \mathcal L_t}{\partial h_t^{\mathrm{out}}}
\right)
=
\frac{1}{|\mathcal Y_m|}
\sum_{t\in\mathcal Y_m}
W_{\mathrm{LM}}^\top(e_{y_t}-p_t).
\label{eq:correction_direction}
\end{equation}
Here, $\mathcal Y_m$ denotes the target-output token positions, $h_t^{\mathrm{out}}\in\mathbb R^{H}$ is the final hidden state before the LM head at position $t$, $W_{\mathrm{LM}}$ is the LM-head projection matrix, $p_t=\mathrm{softmax}(W_{\mathrm{LM}}h_t^{\mathrm{out}})$ is the predicted token distribution, $y_t$ is the ground-truth token, $e_{y_t}$ is its one-hot vector, and $\mathcal L_t=-\log p_t[y_t]$ is the token-level cross-entropy loss. 
The vector $R_m$ represents the supervised correction signal exposed by this experience: it specifies the local hidden-space direction that would increase the likelihood of the target output under the current model. 
Unlike a text summary or a semantic embedding, this signal captures the functional training value of the experience. 
We therefore store its normalized form $d_m=\mathrm{Norm}(R_m)$ as a reusable adaptation direction in memory.
This direction is first-order optimal for locally decreasing the supervised loss of the experience in the pre-LM-head hidden space (Appendix~\ref{app:theory}, Proposition~\ref{prop:correction_optimality}).

\subsubsection{Budget-Aware Memory Management}
\label{sec:memory_writing}

Given the semantic key $k_m$ and correction direction $d_m$ extracted from a new experience, \sysname{} writes it into the memory bank by first updating the direction prototypes and then repairing the memory if the storage budget is exceeded.

\textbf{Prototype assignment and update.}
If $|\mathcal{B}|=0$, \sysname{} directly creates a new prototype $u_{\mathrm{new}}=(d_m,\{k_m\},1)$. 
Otherwise, it finds the closest existing prototype in the correction-direction space:
\[
j^\star=\arg\max_{1\le j\le J}\cos(d_m,d_j).
\]
Let $\tau_d$ be the direction-merge threshold. 
If $\cos(d_m,d_{j^\star})<\tau_d$, the experience is not sufficiently aligned with any existing prototype, so it is treated as a new direction and stored as $u_{\mathrm{new}}=(d_m,\{k_m\},1)$. 
Otherwise, it is merged into $u_{j^\star}$ by updating the prototype direction with a running average:
\begin{equation}
d_{j^\star}\leftarrow
\mathrm{Norm}
\left(
n_{j^\star}d_{j^\star}+d_m
\right),
\qquad
n_{j^\star}\leftarrow n_{j^\star}+1.
\label{eq:prototype_update}
\end{equation}

After merging into $u_{j^\star}$, \sysname{} uses $k_m$ only to expand the semantic access points of the assigned prototype. 
With semantic-anchor novelty threshold $\tau_k$, we update
\begin{equation}
\mathcal A_{j^\star}\leftarrow \mathcal A_{j^\star}\cup\{k_m\}
\quad
\text{if }
\max_{a\in\mathcal A_{j^\star}}\cos(k_m,a)<\tau_k.
\label{eq:anchor_update}
\end{equation}
Otherwise, $k_m$ is discarded as a redundant anchor. 
This enables one correction direction to be retrieved from diverse contexts without storing redundant anchors or raw examples.

\textbf{Memory repair.}
When the memory footprint exceeds $C_{\mathrm{mem}}$, \sysname{} repairs memory in a direction-preserving order: it first prunes redundant semantic anchors, then merges the most similar prototypes with minimal distortion, and finally evicts the prototype with the lowest value:
\begin{equation}
\mathrm{value}(u_j)
=
\frac{
\log(1+n_j)
\left(
1-\max_{r\ne j}\cos(d_j,d_r)
\right)
}{
\mathrm{Size}(u_j)
}.
\label{eq:prototype_value}
\end{equation}
This value favors prototypes that are well supported, directionally distinctive, and storage-efficient. 
As formalized by the directional coverage view at the end of Section~\ref{sec:inference_adaptation}, the benefit of memory depends on how well stored directions cover historical corrections; thus, the limited budget is allocated to diverse reusable correction directions rather than raw examples.

\subsubsection{Inference-Time Retrieval and Adaptation}
\label{sec:inference_adaptation}

After continual training, given a query $q$ at inference time, \sysname{} first constructs query-side retrieval signals, then retrieves relevant historical correction directions, and finally applies the retrieved direction through temporary hidden-state adaptation to locally improve the current prediction without permanently modifying the model.

\textbf{Query-side retrieval signals.}
\sysname{} computes the query semantic key $k_q$ by applying the same pooling function in Eq.~\eqref{eq:semantic_key} to the query token positions. 
This key provides the semantic signal for accessing memory. 
Since the ground-truth output is unavailable at inference time, \sysname{} also constructs an unsupervised directional proxy from local prediction uncertainty. 
Let $h_q$ be the pre-LM-head hidden state at the answer-start or first probe decoding position, and let $p_q=\mathrm{softmax}(W_{\mathrm{LM}}h_q)$. 
We keep the top-$K_v$ token set $\mathcal V_q$ under $p_q$ and renormalize their probabilities as $\hat p_q(v)=p_q(v)/\sum_{v'\in\mathcal V_q}p_q(v')$ for $v\in\mathcal V_q$. 
The restricted entropy is then $\hat H_q=-\sum_{v\in\mathcal V_q}\hat p_q(v)\log \hat p_q(v)$.
The gradient of $\hat H_q$ with respect to the restricted logits is
\begin{equation}
g_q
=
-\hat p_q\odot
\left(
\log \hat p_q+\hat H_q\mathbf 1
\right).
\label{eq:g_p}
\end{equation}
Projecting the restricted entropy gradient back to the hidden space gives
\begin{equation}
\tilde d_q=\mathrm{Norm}(-r_q),
\qquad
\text{where }
r_q=W_{\mathcal V_q}^{\top}g_q.
\label{eq:uncertainty_direction}
\end{equation}
If the norm of $r_q$ exceeds a stability threshold $\epsilon_r$, \sysname{} uses $\tilde d_q$ as the entropy-decreasing query direction. 
This direction is the steepest first-order direction for decreasing the restricted predictive entropy $\hat H_q$ at the current query state (Appendix~\ref{app:theory}, Proposition~\ref{prop:uncertainty_proxy}), and therefore serves as a label-free geometric cue for retrieving historical directions that may affect the current prediction.

\textbf{Memory retrieval and direction aggregation.}
Given $k_q$ and, when available, $\tilde d_q$, \sysname{} retrieves memory prototypes from both semantic and directional channels. 
For each prototype $u_j$, its semantic relevance is measured by the best matching anchor, $S_k(q,j)=\max_{a\in\mathcal A_j}\cos(k_q,a)$. 
The candidate set is defined as
\begin{equation}
\mathcal C_q
=
\operatorname{Top}_{R_k}\!\left(S_k(q,j)\right)
\cup
\begin{cases}
\operatorname{Top}_{R_d}\!\left(\cos(\tilde d_q,d_j)\right), 
& \text{if } \tilde d_q \text{ is available},\\
\varnothing, 
& \text{otherwise}.
\end{cases}
\label{eq:candidate_set}
\end{equation}
Here, $\operatorname{Top}_{R_k}$ and $\operatorname{Top}_{R_d}$ return the prototype indices with the highest semantic and directional scores, respectively.
For each candidate $u_j$ with $j\in\mathcal C_q$, \sysname{} computes
\begin{equation}
\mathrm{score}(q,j)
=
\lambda_k S_k(q,j)
+
\lambda_u U(q,j),
\qquad
\text{where }U(q,j)=
\max
\left(
0,
-
g_q^\top
\epsilon W_{\mathcal V_q}d_j
\right).
\label{eq:retrieval_score}
\end{equation}
Here, $U(q,j)$ estimates whether moving along $d_j$ would locally reduce the query uncertainty, $\epsilon$ is a small probe scale, and $\lambda_k,\lambda_u\ge 0$ balance semantic relevance and directional utility. 
If the directional proxy is unavailable, we set $U(q,j)=0$.
The top-$M$ candidates under $\mathrm{score}(q,j)$ are selected as $S_q$ and aggregated into a query-specific correction direction $\bar d_q$:
\begin{equation}
\bar d_q=
\mathrm{Norm}
\left(
\sum_{j\in S_q}w_jd_j
\right),
\qquad
\text{where }w_j=
\frac{
\exp(\mathrm{score}(q,j)/T_{\mathrm{agg}})
}{
\sum_{r\in S_q}\exp(\mathrm{score}(q,r)/T_{\mathrm{agg}})
}.
\label{eq:direction_aggregation}
\end{equation}
Here, $T_{\mathrm{agg}}$ is the softmax temperature. 
The resulting $\bar d_q$ summarizes the historical correction directions retrieved for the current query.

\textbf{Temporary hidden adaptation.}
The aggregated direction $\bar d_q$ is applied during generation through a query-specific hidden update. 
Rather than applying the same correction uniformly to all decoding states, \sysname{} uses an activation gate to make the correction stronger when the current generation state is close to the query state that triggered retrieval:
\begin{equation}
p_t'=
\mathrm{softmax}(W_{\mathrm{LM}}(h_t^{\mathrm{out}}
+
\eta \alpha_t\bar d_q)),
\qquad
\text{where }\alpha_t=
\max(0,\cos(\mathrm{Norm}(h_q),\mathrm{Norm}(h_t^{\mathrm{out}}))).
\label{eq:hidden_adaptation}
\end{equation}
Here, $\eta$ is the adaptation scale. 
Both the direction aggregation in Eq.~\eqref{eq:direction_aggregation} and the activation gate in Eq.~\eqref{eq:hidden_adaptation} admit a local least-squares interpretation (Appendix~\ref{app:theory}, Proposition~\ref{prop:hidden_adapter_lsq}). 
An ablation in Appendix~\ref{app:gate_ablation} further confirms that this activation gate is necessary, as uniformly applying the retrieved direction to all decoding states leads to worse performance and stronger forgetting.
Thus, the inference-time update can be viewed as a low-cost rank-one correction in hidden space, modulated by local state similarity.

\begin{informalthm}
Suppose the supervised loss is $\beta$-smooth in the pre-LM-head hidden space. 
For an old-stage example $z$, applying inference-time hidden adaptation with retrieved direction $\bar d_z$ reduces the old-stage risk to first order by a term proportional to
\[
\eta \|G_z^\alpha\|_2 \langle b_z,\bar d_z\rangle - \frac{\beta\eta^2}{2},
\]
where $G_z^\alpha$ is the gated old-task correction vector and $b_z=G_z^\alpha/\|G_z^\alpha\|_2$. 
Thus, the benefit of memory is controlled by its \emph{directional coverage} of historical corrections, rather than by raw example coverage. 
The full statement and proof are given in Appendix~\ref{app:theory}, Theorem~\ref{thm:directional_coverage}.
\end{informalthm}

This bound clarifies the role of memory in \sysname{}: stored prototypes should cover diverse historical correction directions that can reduce old-stage risk after retrieval and adaptation. 
This provides a unified theoretical view of the design choices, including direction prototypes in Section~\ref{sec:memory_unit}, direction-preserving memory repair in Section~\ref{sec:memory_writing}, and temporary hidden adaptation.
The update is discarded after generation, enabling local adaptation without modifying the backbone, updating the LM head, extending the prompt with retrieved examples, or running test-time backpropagation. 
This is especially suitable for SLM CL, where repeatedly writing new experiences into a capacity-limited backbone can overwrite prior capabilities.
An inference-time complexity analysis is in Appendix~\ref{app:inference_efficiency}.

\section{Experiments}
\label{sec:exp}
\subsection{Experimental Setup}





\textbf{Benchmarks.}
We evaluate our method on four supervised CL benchmarks covering classification, question answering (QA), and generation tasks: 
(i) \textbf{class-incremental learning (CIL)} on \textsc{Banking77}~\cite{Banking77}, an intent classification dataset with 77 classes, split into 7 sequential tasks with 11 new intents per task; 
(ii) \textbf{domain-incremental learning (DIL)} on two benchmarks: \textsc{Document Sentiment Classification} (\textsc{DSC})~\cite{DSC}, which contains 10 product-review domains with a fixed positive/negative label space, and \textsc{PAXQA}~\cite{PAXQA}, a multilingual QA dataset where 4 languages are treated as sequential domains; 
and (iii) \textbf{task-incremental learning (TIL)} on the \textsc{SuperNI Benchmark}~\cite{SuperNI}, a 15-task instruction-following benchmark covering dialogue generation, information extraction, QA, summarization, and sentiment classification. 

\textbf{Baselines.}
We compare \sysname{} with three groups of baselines:
(i) general language model CL baselines, including LoRA~\cite{LoRA} and MBPA++~\cite{MBPA};
(ii) memory-based LLM CL baselines, including CIS~\cite{CIS}, InCA~\cite{InCA}, CT0~\cite{CT0}, and FOREVER~\cite{FOREVER};
and (iii) SLM-oriented CL baselines, including DKVB~\cite{DKVB} and Srinath K \textit{et al.}~\cite{POS}.
Implementation details and the adaptation of each baseline to our supervised SLM CL setting are provided in Appendix~\ref{app:baselines}.

\textbf{Evaluation Metrics.}
We evaluate CL performance using a performance matrix $\mathbf{A}\in\mathbb{R}^{T\times T}$, where $A_{t,\tau}$ denotes the task-specific score on task $\tau$ after the model has learned task $t$. 
For classification tasks, $A_{t,\tau}$ is accuracy; for QA tasks, $A_{t,\tau}$ is exact match; and for generation tasks in \textsc{SuperNI}, $A_{t,\tau}$ is Rouge-L. 
All scores are scaled to $[0,100]$. 
We report \textbf{Overall Performance (OP)} as $\mathrm{OP}=\frac{1}{T}\sum_{\tau=1}^{T}A_{T,\tau}$, which measures the final average task score over all tasks, and \textbf{Backward Transfer (BWT)} as $\mathrm{BWT}=\frac{1}{T-1}\sum_{\tau=1}^{T-1}(A_{T,\tau}-A_{\tau,\tau})$, which measures how learning later tasks affects earlier tasks. 
Higher OP and larger BWT are better.

\textbf{Implementation Details.}
We evaluate all methods on instruction-tuned base models, including Qwen3 (0.6B, 1,7B, 4B)~\cite{Qwen3} and LLaMA (3.2-3B, 3.1-8B)~\cite{Llama}.
Unless otherwise specified, we use a fixed 1\% memory budget, measured as the byte size of 1\% of the full training stream.
For \sysname{}, we set the direction-merge threshold to $\tau_d=0.85$ and the semantic-anchor novelty threshold to $\tau_k=0.80$. 
A sensitivity analysis in Appendix~\ref{app:hyperparameter_sensitivity} shows that \sysname{} remains stable across both thresholds.
During inference, we retrieve $R_k=R_d=8$ candidates and aggregate the top $M=4$ prototypes after reranking with weights $\lambda_k=1.0$ and $\lambda_u=0.5$, using $T=0.2$ and $\eta=0.5$. 
Results average three independent runs. 
Algorithm~\ref{alg:miita} in Appendix~\ref{app:miita_algorithm} details the implementation.

\subsection{Overall Performance}

\begin{table*}[t]
\centering
\footnotesize
\setlength{\tabcolsep}{3pt}
\caption{
Overall continual learning performance on four datasets.
\textbf{Bold} indicates the best result, \underline{underline} indicates the second-best result, and $\uparrow$ indicates that higher values are better.
}
\label{tab:overall_performance_scaling}
\begin{tabular}{clcc|cc|cc|cc|cc}
\toprule
\multirow{2.5}{*}{\makecell[c]{\textbf{Base}\\\textbf{Model}}}
& \multirow{2.5}{*}{\textbf{Method}} 
& \multirow{2.5}{*}{\makecell[c]{\textbf{Memory-}\\\textbf{Based}}}
& \multirow{2.5}{*}{\makecell[c]{\textbf{Designed}\\\textbf{For}}}
& \multicolumn{2}{c}{\textsc{Banking77}} 
& \multicolumn{2}{|c|}{\textsc{DSC}} 
& \multicolumn{2}{c}{\textsc{PAXQA}} 
& \multicolumn{2}{|c}{\textsc{SuperNI}} \\
\cmidrule(lr){5-6}\cmidrule(lr){7-8}\cmidrule(lr){9-10}\cmidrule(lr){11-12}
& & & & OP$\uparrow$ & BWT$\uparrow$ & OP$\uparrow$ & BWT$\uparrow$ & OP$\uparrow$ & BWT$\uparrow$ & OP$\uparrow$ & BWT$\uparrow$ \\
\midrule

\multirow{9.5}{*}{\rotatebox[origin=c]{90}{\textbf{Qwen3-0.6B}}}
& LoRA~\cite{LoRA} 
& \xmark & -- 
& 65.31 & -11.44 
& 66.04 & -11.53 
& 40.75 & -5.57 
& 31.24 & -8.91 \\

& MBPA++~\cite{MBPA} 
& \cmark & LM 
& 67.82 & -10.22 
& 67.22 & -9.12 
& 40.48 & -6.12 
& 31.58 & -10.21 \\

\cmidrule(lr){2-12}

& CIS~\cite{CIS} 
& \cmark & LLM 
& 66.78 & -10.60 
& 65.47 & -10.66 
& 38.96 & -6.74 
& 26.51 & -11.05 \\

& InCA~\cite{InCA} 
& \cmark & LLM 
& 71.20 & -8.72 
& 66.79 & -10.25 
& 39.47 & -6.17 
& 28.58 & -11.17 \\

& CT0~\cite{CT0} 
& \cmark & LLM 
& 68.79 & -9.78 
& 67.71 & -9.75 
& 41.46 & -5.75 
& 30.43 & -10.82 \\

& FOREVER~\cite{FOREVER} 
& \cmark & LLM 
& 70.58 & -9.14 
& 68.12 & -9.63 
& 42.54 & -5.24 
& 32.82 & -9.95 \\

\cmidrule(lr){2-12}

& DKVB~\cite{DKVB} 
& \xmark & SLM 
& \underline{73.12} & \underline{-7.43} 
& \underline{75.47} & \underline{-7.23} 
& 53.45 & -4.64 
& \underline{37.89} & \underline{-6.72} \\

& Srinath K \textit{et al.}~\cite{POS} 
& \cmark & SLM 
& 70.44 & -7.92 
& 69.42 & -9.24 
& \underline{55.64} & \underline{-4.43} 
& 35.54 & -7.44 \\

& \cellcolor{gray!20}\textbf{\sysname{} (ours)} 
& \cellcolor{gray!20}\cmark 
& \cellcolor{gray!20}SLM 
& \cellcolor{gray!20}\textbf{77.67} 
& \cellcolor{gray!20}\textbf{-6.17} 
& \cellcolor{gray!20}\textbf{78.29} 
& \cellcolor{gray!20}\textbf{-6.14} 
& \cellcolor{gray!20}\textbf{60.01} 
& \cellcolor{gray!20}\textbf{-3.12} 
& \cellcolor{gray!20}\textbf{40.58} 
& \cellcolor{gray!20}\textbf{-4.89} \\

\midrule

\multirow{9.5}{*}{\rotatebox[origin=c]{90}{\textbf{Qwen3-4B}}}
& LoRA~\cite{LoRA} 
& \xmark & -- 
& 68.42 & -9.85 
& 69.38 & -9.74 
& 44.92 & -4.82 
& 34.17 & -7.96 \\

& MBPA++~\cite{MBPA} 
& \cmark & LM 
& 71.36 & -8.71 
& 70.54 & -7.82 
& 44.15 & -5.05 
& 36.35 & -8.41 \\

\cmidrule(lr){2-12}

& CIS~\cite{CIS} 
& \cmark & LLM 
& 71.84 & -8.88 
& 69.75 & -8.92 
& 43.05 & -5.73 
& 38.40 & -8.92 \\

& InCA~\cite{InCA} 
& \cmark & LLM 
& \underline{76.48} & -6.95 
& 70.95 & -8.76 
& 43.74 & -5.28 
& 40.12 & -8.64 \\

& CT0~\cite{CT0} 
& \cmark & LLM 
& 73.92 & -7.78 
& 71.86 & -8.15 
& 46.18 & -4.52 
& 42.68 & -8.37 \\

& FOREVER~\cite{FOREVER} 
& \cmark & LLM 
& 75.10 & -7.21 
& 72.64 & -7.65 
& 47.32 & -4.12 
& 44.00 & -8.05 \\

\cmidrule(lr){2-12}

& DKVB~\cite{DKVB} 
& \xmark & SLM 
& 76.22 & \underline{-6.31} 
& \underline{78.10} & \underline{-6.05} 
& 57.82 & -3.96 
& \underline{44.76} & \underline{-6.08} \\

& Srinath K \textit{et al.}~\cite{POS} 
& \cmark & SLM 
& 73.18 & -6.94 
& 72.68 & -7.86 
& \underline{61.47} & \underline{-3.62} 
& 41.85 & -6.91 \\

& \cellcolor{gray!20}\textbf{\sysname{} (ours)} 
& \cellcolor{gray!20}\cmark 
& \cellcolor{gray!20}SLM 
& \cellcolor{gray!20}\textbf{81.05} 
& \cellcolor{gray!20}\textbf{-5.02} 
& \cellcolor{gray!20}\textbf{81.66} 
& \cellcolor{gray!20}\textbf{-4.95} 
& \cellcolor{gray!20}\textbf{65.21} 
& \cellcolor{gray!20}\textbf{-2.38} 
& \cellcolor{gray!20}\textbf{48.21} 
& \cellcolor{gray!20}\textbf{-4.27} \\

\bottomrule
\end{tabular}
\end{table*}

Table~\ref{tab:overall_performance_scaling} reports the main results for Qwen3-0.6B and Qwen3-4B under a 1\% persistent memory budget. \sysname{} consistently achieves the best OP and BWT across all benchmarks. This simultaneous improvement confirms that \sysname{} genuinely enhances the stability-plasticity trade-off, successfully incorporating new experiences without simply overfitting later stages at the expense of earlier ones.
Crucially, \sysname{} overcomes the fundamental limitations of both parameter-update and textual-memory paradigms. Parameter-updating baselines (e.g., LoRA) suffer from severe forgetting, while SLM-oriented methods show only task-dependent gains, confirming that repeatedly modifying limited SLM capacity is unreliable. Similarly, memory-based LLM methods struggle under strict storage budgets, especially on the smaller 0.6B backbone and open-ended tasks. This validates our intuition that textual memory is inefficient for SLMs, as it consumes surface-form storage and demands strong in-context reasoning. 
Overall, \sysname{}'s consistent superiority validates its core design. Rather than destructively overwriting parameters or relying on surface textual prompts, \sysname{} preserves compact functional correction directions. Applying these directions via disposable inference-time hidden adaptation provides a more direct, reusable, and non-destructive mechanism for continual learning.

\subsection{Model-Size Scaling under a Fixed Memory Budget}
\label{sec:model_size_scaling}

\begin{wrapfigure}{R}{0.7\textwidth}
    \centering
    \includegraphics[width=0.7\textwidth]{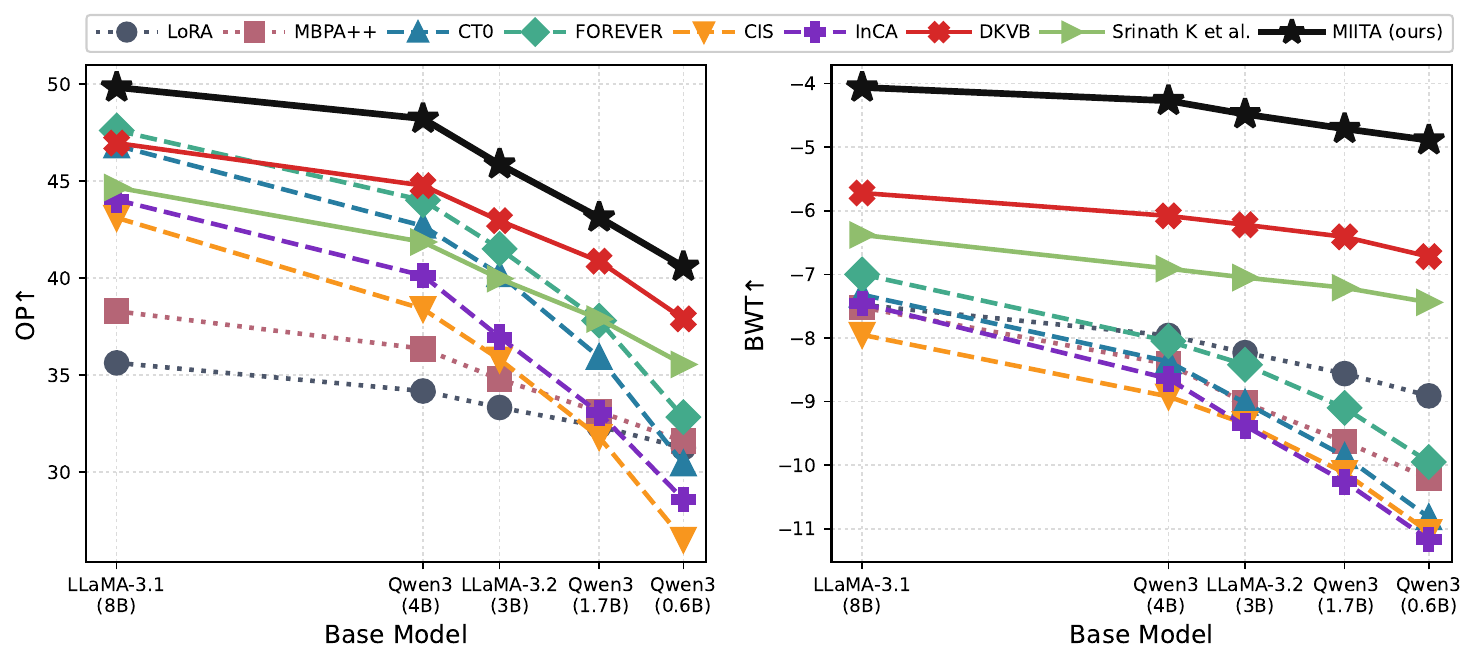}
    \caption{
    Model-size scaling on \textsc{SuperNI} under 1\% memory budget. 
    }
    \label{fig:superni_scaling}
\end{wrapfigure}
To investigate whether existing methods remain effective as the base model becomes smaller under the same memory budget, we conduct model scaling experiments, as shown in Figure~\ref{fig:superni_scaling}.  The results reveal a severe degradation in memory methods designed for LLMs. Specifically, replay-based methods (CT0, FOREVER) drop substantially, highlighting their reliance on large parameter capacity to absorb and retain knowledge. Similarly, prompting-based methods (CIS, InCA) falter because SLMs lack the strong in-context reasoning required to utilize textual memories. While SLM-oriented baselines (DKVB, Srinath K \textit{et al.}) exhibit better relative robustness, their performance still declines, indicating that purely architectural or replay-adapter solutions are insufficient. In contrast, \sysname{} consistently achieves the best OP and BWT across all scales. By preserving compact functional correction directions rather than raw text or examples, \sysname{} ensures that historical memory remains highly effective despite limited model capacity and tight storage budgets.

\subsection{Memory Budget Sensitivity}
\label{sec:budget_sensitivity}

\begin{wrapfigure}{R}{0.48\textwidth}
    \centering
    \includegraphics[width=0.48\textwidth]{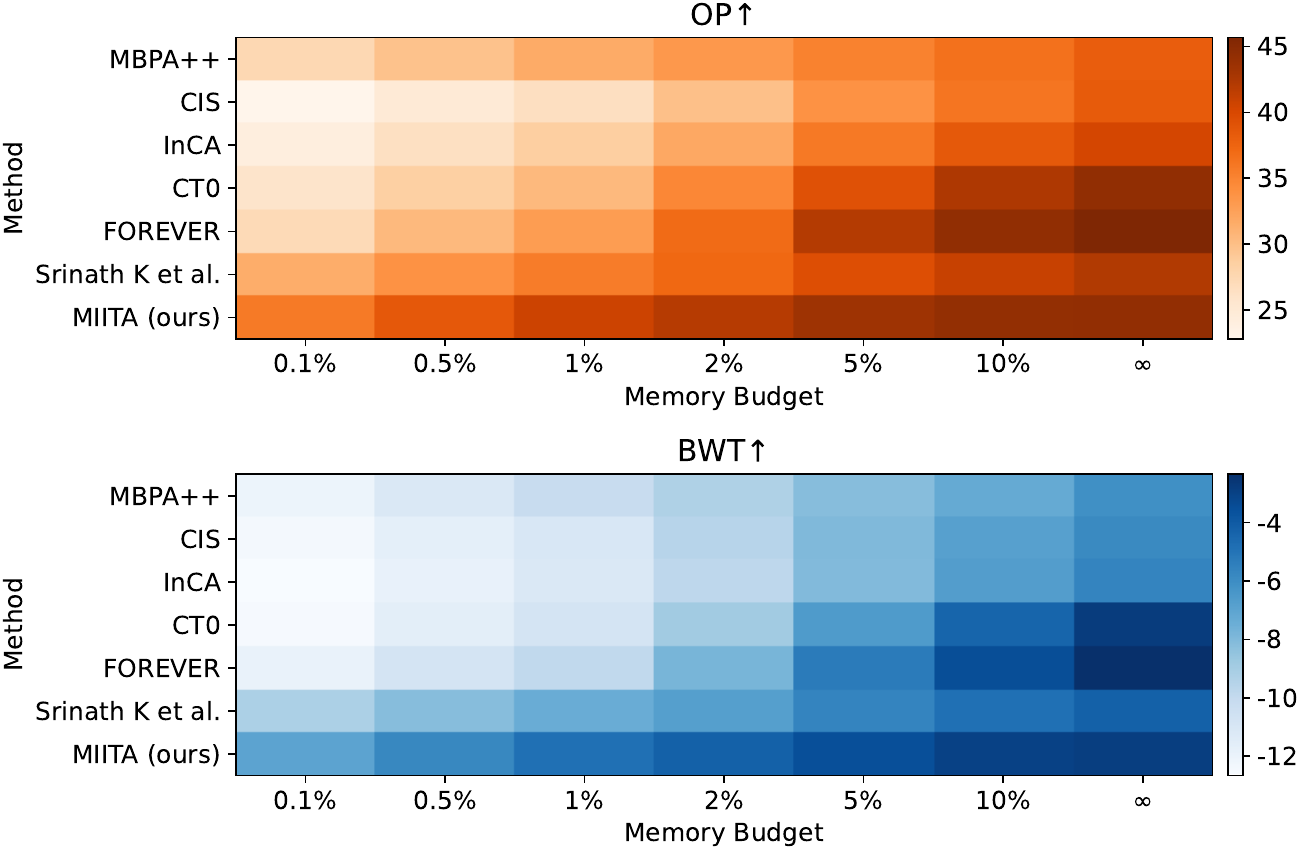}
    \caption{
    Budget sensitivity on \textsc{SuperNI} with Qwen3-0.6B for memory-based methods. 
    Each column denotes a fixed auxiliary memory budget; $\infty$ denotes the unlimited-memory setting. 
    }
    \label{fig:budget_sensitivity}
\end{wrapfigure}

Figure~\ref{fig:budget_sensitivity} evaluates memory-based CL methods under increasing storage budgets, focusing on those that scale with external memory and omitting fixed-storage baselines such as LoRA and DKVB. The results show that existing memory-based LLM CL methods are highly sensitive to available storage: replay-based methods like CT0 and FOREVER improve substantially as the budget increases, with the latter becoming strongest in the practically unlimited-memory setting by leveraging large buffers to retain extensive raw supervision from previous tasks. However, their performance drops sharply under tight budgets, indicating that replaying old examples is inefficient when memory is minimal. In contrast, \sysname{} remains consistently strong in the low-budget regime (0.1\% to 2\%), where the gap between our approach and memory methods designed for LLMs is most pronounced. This supports our central motivation: for SLM continual learning, the key challenge is not merely whether memory is used, but whether the stored representation is compact and functionally useful. 
While methods storing raw examples, summaries, or prompt contexts require significantly larger budgets to be competitive, \sysname{} stores supervised correction directions that directly encode reusable update signals.

\subsection{Effect of Functional Correction Directions}
\label{sec:functional_memory}

A key question in \sysname{} is whether supervised experiences should be stored as textual or semantic content, or as functional correction directions that encode their effect on the model prediction.

\begin{wrapfigure}{R}{0.65\textwidth}
    \centering
    \begin{minipage}[t]{0.32\textwidth}
        \centering
        
        \includegraphics[width=1\linewidth]{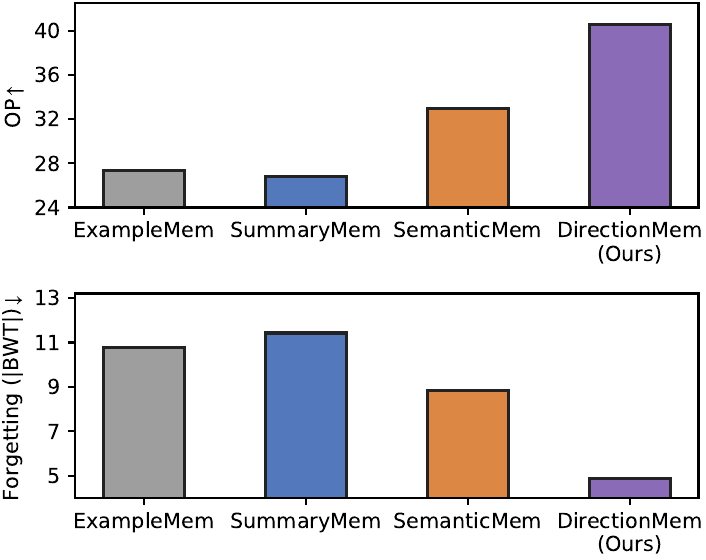}
        \\(a)
    \end{minipage}
    \hfill
    \begin{minipage}[t]{0.32\textwidth}
        \centering
        
        \includegraphics[width=1\linewidth]{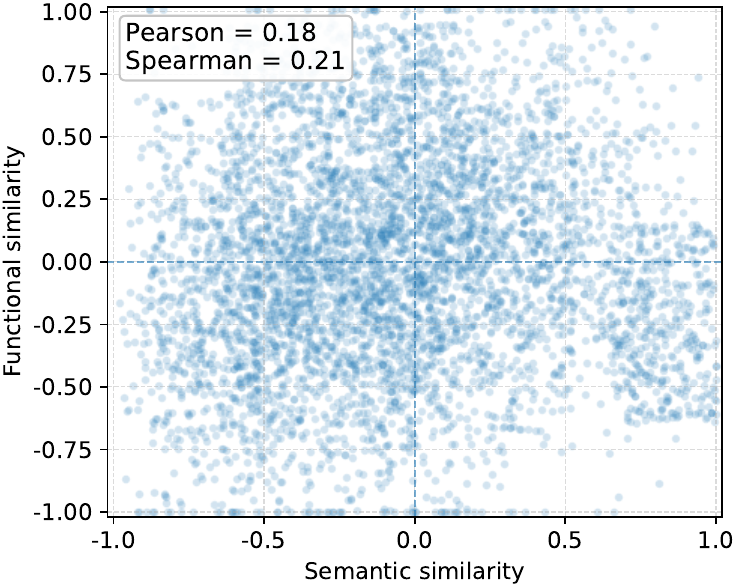}
        \\(b)
    \end{minipage}
    \caption{
    Analysis of why \sysname{} stores functional correction directions on \textsc{SuperNI}.
    \textbf{(a)} Comparison of memory variants with Qwen3-0.6B under a fixed 1\% memory budget.
    \textbf{(b)} Pairwise comparison between semantic and functional similarity, where each point represents a pair of training examples.
    }
    \label{fig:functional_analysis}
\end{wrapfigure}
\textbf{Memory variant comparison.}
To first isolate the effect of what is stored in memory, we compare four variants differing exclusively in their memory content: \textbf{ExampleMem} stores raw supervised examples $(x,y)$, \textbf{SummaryMem} stores textual summaries of historical examples, \textbf{SemanticMem} stores semantic embedding prototypes, and \textbf{DirectionMem} stores correction-direction prototypes induced by supervised examples. 
Implementation details are provided in Appendix~\ref{app:memory_item_variants}. 
As shown in Figure~\ref{fig:functional_analysis}(a), ExampleMem and SummaryMem are limited by the weak in-context utilization of small models, since they require the SLM to reuse retrieved textual contexts. 
SemanticMem avoids long textual memories and is more compact, but its semantic-redundancy-based compression does not reliably preserve the functional update signals needed under a tight storage budget. 
In contrast, DirectionMem directly stores correction-direction prototypes, preserving how past supervised examples should change the model's prediction. 
Its superior OP and forgetting results show that \sysname{} benefits from capturing functional redundancy, rather than merely storing human-readable text or semantic similarity.

\textbf{Semantic vs. functional redundancy.}
To further validate the divergence between semantic and functional redundancy, we compare pairwise cosine similarity among training examples, using semantic keys for semantic similarity and correction directions for functional similarity. 
The two measures are only weakly correlated, with Pearson correlation 0.18 and Spearman correlation 0.21, as visualized in Figure~\ref{fig:functional_analysis}(b). 
The Pearson correlation corresponds to only about $R^2\approx 0.03$, meaning that semantic similarity explains less than 5\% of the variance in functional similarity. 
Consistent with this, the points are widely scattered rather than concentrated along an increasing diagonal trend: semantically similar examples can still require very different correction directions, while semantically distant examples may share useful functional updates. 
This suggests that purely semantic compression can miss many functionally important directions, whereas direction-level prototypes preserve supervised update signals according to their functional effects.
\subsection{Validating Uncertainty-Guided Retrieval}
\label{sec:uncertainty_proxy}

\begin{wraptable}{R}{0.405\textwidth}
\centering
\footnotesize
\setlength{\tabcolsep}{5pt}
\caption{
Validation of uncertainty-guided retrieval on \textsc{SuperNI} with Qwen3-0.6B under 1\% memory budget. OracleRet is the label-informed upper bound. 
}
\label{tab:uncertainty_proxy}
\begin{tabular}{lccc}
\toprule
\textbf{Retrieval Variant} 
& OP$\uparrow$ 
& BWT$\uparrow$ 
& OA$\uparrow$ \\
\midrule
RandomRet
& 28.85 
& -12.63 
& 0.03 \\

SemanticRet
& 36.21 
& -6.98 
& 0.35 \\

UncertaintyRet
& 34.74 
& -8.92 
& 0.42 \\

\rowcolor{gray!20}SemUncRet (Ours)
& \textbf{40.58} 
& \textbf{-4.89} 
& \textbf{0.52} \\
\midrule
OracleRet
& 47.29
& -2.42
& -- \\
\bottomrule
\end{tabular}
\end{wraptable}
To evaluate \sysname{}'s retrieval design by ablating inference-time query signals, we compare variants differing exclusively in how prototypes are retrieved: \textbf{RandomRet} uniformly samples prototypes; \textbf{SemanticRet} relies solely on semantic-key similarity $S_k(q,j)$; \textbf{UncertaintyRet} scores stored directions $d_j$ using only the uncertainty-induced direction; \textbf{SemUncRet} employs the full \sysname{} strategy; and \textbf{OracleRet} provides an upper bound using a label-derived oracle correction direction.
To quantify consistency with supervised corrections, we report \textbf{Oracle Alignment (OA)}, the average cosine alignment between the retrieved and oracle directions.
Implementation details of the retrieval variants and OA are provided in Appendix~\ref{app:retrieval_variants}.
Table~\ref{tab:uncertainty_proxy} confirms uncertainty is an effective query-side proxy. UncertaintyRet outperforms RandomRet and achieves higher OA than SemanticRet, proving it provides a meaningful label-free geometric cue. However, UncertaintyRet yields lower OP and BWT than SemanticRet because uncertainty captures local geometry but lacks semantic relevance. The optimal variant, SemUncRet, resolves this: semantic recall ensures relevance, while uncertainty prioritizes correction utility. This validates \sysname{}'s design: uncertainty is a powerful label-free proxy, maximizing effectiveness when combined with semantic recall for reranking.

\section{Conclusion}
\label{sec:conclusion}
We introduced \sysname{}, a memory-induced inference-time adaptation framework for supervised continual learning (CL) in small language models (SLMs) under constrained persistent storage. 
Unlike replay-based or prompting-based memory methods, \sysname{} preserves historical experiences as compact functional correction directions rather than raw examples, textual summaries, or long contexts. 
These directions are organized into budget-aware prototypes with lightweight semantic anchors and are invoked at inference time through semantic and uncertainty-guided retrieval followed by gated temporary hidden-state adaptation. 
This enables non-destructive reuse of past supervision without backbone updates, prompt extension, or test-time backpropagation. 
Experiments across diverse CL settings show that \sysname{} improves final performance and mitigates forgetting under fixed memory budgets, while further analyses demonstrate its robustness across model scales and memory budgets and validate the roles of functional correction memory and uncertainty-guided retrieval. 
Overall, this work suggests that SLM CL can benefit from shifting memory design away from storing textual content toward storing reusable functional adaptation signals.

\bibliographystyle{abbrv}
\bibliography{refer}

\newpage
\appendix
\section{Theoretical Analysis}
\label{app:theory}

We provide a local theoretical analysis of \sysname{} to justify the main components used in Section~\ref{sec:miita_overview}: correction-direction extraction, uncertainty-guided retrieval, temporary hidden adaptation, and direction-memory coverage. 
The analysis does not assume global convexity or task stationarity; it only uses differentiability of the LM-head loss and first-order local approximations.

\paragraph{Notation.}
For a target token $y$ with pre-LM-head hidden state $h\in\mathbb R^H$, define
\[
\ell_y(h)=-\log \mathrm{softmax}(W_{\mathrm{LM}}h)[y].
\]
Let $p=\mathrm{softmax}(W_{\mathrm{LM}}h)$. Then
\[
\nabla_h \ell_y(h)=W_{\mathrm{LM}}^\top(p-e_y),
\qquad
-\nabla_h \ell_y(h)=W_{\mathrm{LM}}^\top(e_y-p).
\]
For an experience $m$, define its hidden-perturbed loss
\[
\mathcal L_m(\delta)
=
\frac{1}{|\mathcal Y_m|}
\sum_{t\in\mathcal Y_m}
\ell_{y_t}(h_t^{\mathrm{out}}+\delta).
\]
Then the correction vector used by \sysname{} is exactly
\[
R_m=-\nabla_\delta \mathcal L_m(0).
\]

\begin{proposition}[First-order optimality of correction directions]
\label{prop:correction_optimality}
Assume $R_m\neq 0$ and let $d_m=R_m/\|R_m\|_2$. 
For any unit direction $d$ and sufficiently small $\eta>0$,
\[
\mathcal L_m(\eta d)
=
\mathcal L_m(0)
-
\eta \langle R_m,d\rangle
+
o(\eta).
\]
Therefore, among all unit-norm hidden-space directions, $d_m$ maximizes the first-order decrease of $\mathcal L_m$:
\[
d_m
=
\arg\max_{\|d\|_2=1}
\langle R_m,d\rangle.
\]
\end{proposition}

\begin{proof}
By the first-order Taylor expansion of $\mathcal L_m$ around $\delta=0$,
\[
\mathcal L_m(\eta d)
=
\mathcal L_m(0)
+
\eta \langle \nabla_\delta \mathcal L_m(0),d\rangle
+
o(\eta).
\]
Since $R_m=-\nabla_\delta \mathcal L_m(0)$, we obtain
\[
\mathcal L_m(\eta d)
=
\mathcal L_m(0)
-
\eta \langle R_m,d\rangle
+
o(\eta).
\]
The first-order loss decrease is therefore proportional to $\langle R_m,d\rangle$. 
By Cauchy--Schwarz,
\[
\langle R_m,d\rangle
\le
\|R_m\|_2\|d\|_2
=
\|R_m\|_2,
\]
with equality if and only if $d=R_m/\|R_m\|_2=d_m$. 
Hence $d_m$ is the steepest first-order correction direction in the pre-LM-head hidden space.
\end{proof}

\begin{proposition}[Uncertainty direction as a local geometric proxy]
\label{prop:uncertainty_proxy}
Fix the restricted vocabulary set $\mathcal V_q$ at query state $h_q$. Let
\[
\hat p_q(v)
=
\frac{p_q(v)}
{\sum_{v'\in\mathcal V_q}p_q(v')},
\qquad
v\in\mathcal V_q,
\]
and
\[
\hat H_q
=
-\sum_{v\in\mathcal V_q}
\hat p_q(v)\log \hat p_q(v).
\]
Let $\hat z_q=W_{\mathcal V_q}h_q$ be the restricted logits. Then
\[
g_q
=
\frac{\partial \hat H_q}{\partial \hat z_q}
=
-\hat p_q\odot
\left(
\log \hat p_q+\hat H_q\mathbf 1
\right),
\]
and
\[
r_q
=
W_{\mathcal V_q}^\top g_q
=
\nabla_{h_q}\hat H_q.
\]
Therefore, if $r_q\neq 0$, the direction
\[
\tilde d_q
=
-\frac{r_q}{\|r_q\|_2}
\]
is the steepest first-order direction for decreasing the restricted entropy.
\end{proposition}

\begin{proof}
For restricted softmax probabilities $\hat p$, the entropy is
\[
\hat H=-\sum_v \hat p_v\log \hat p_v.
\]
Using the softmax Jacobian
\[
\frac{\partial \hat p_i}{\partial \hat z_j}
=
\hat p_i(\mathbf 1[i=j]-\hat p_j),
\]
we have
\[
\frac{\partial \hat H}{\partial \hat z_j}
=
\sum_i
-(\log \hat p_i+1)
\hat p_i(\mathbf 1[i=j]-\hat p_j).
\]
Rearranging gives
\[
\frac{\partial \hat H}{\partial \hat z_j}
=
-\hat p_j(\log \hat p_j+\hat H).
\]
Thus
\[
g_q
=
-\hat p_q\odot
(\log \hat p_q+\hat H_q\mathbf 1).
\]
Since $\hat z_q=W_{\mathcal V_q}h_q$, the chain rule gives
\[
\nabla_{h_q}\hat H_q
=
W_{\mathcal V_q}^\top g_q
=
r_q.
\]
For any unit direction $d$,
\[
\hat H_q(h_q+\eta d)
=
\hat H_q(h_q)
+
\eta \langle r_q,d\rangle
+
o(\eta).
\]
The first-order entropy decrease is maximized by minimizing $\langle r_q,d\rangle$, whose unit-norm solution is $d=-r_q/\|r_q\|_2$.
\end{proof}

Moreover,
\[
U(q,j)
=
\max
\left(
0,
-g_q^\top \epsilon W_{\mathcal V_q}d_j
\right)
=
\epsilon
\max
\left(
0,
-\langle r_q,d_j\rangle
\right).
\]
Therefore $U(q,j)$ is exactly the positive part of the first-order restricted-entropy decrease induced by moving along historical direction $d_j$.
This proxy does not assume that the entropy-decreasing direction equals the ground-truth correction direction. 
Rather, it exposes the local geometry of the current predictive distribution and is used as a label-free auxiliary retrieval cue for comparing the query state with stored historical directions.

\begin{proposition}[Local least-squares derivation of the hidden adapter]
\label{prop:hidden_adapter_lsq}
Given selected memory directions $\{d_j\}_{j\in S_q}$ and weights $w_j\ge 0$, $\sum_j w_j=1$, consider the problem of representing them by a single unit correction direction:
\[
\min_{\|d\|_2=1}
\sum_{j\in S_q}
w_j\|d-d_j\|_2^2.
\]
If $\sum_j w_j d_j\neq 0$, the solution is
\[
\bar d_q
=
\frac{\sum_{j\in S_q}w_jd_j}
{\left\|\sum_{j\in S_q}w_jd_j\right\|_2}.
\]
Furthermore, with
\[
\hat h_q=\mathrm{Norm}(h_q),
\qquad
\hat h_t=\mathrm{Norm}(h_t^{\mathrm{out}}),
\]
the activation gate
\[
\alpha_t
=
\max(0,\cos(\hat h_t,\hat h_q))
\]
is the nonnegative least-squares coefficient for projecting the current decoding state onto the query state. 
\end{proposition}

\begin{proof}
Expanding the direction-aggregation objective,
\[
\sum_j w_j\|d-d_j\|_2^2
=
\sum_j w_j(\|d\|_2^2+\|d_j\|_2^2-2\langle d,d_j\rangle).
\]
Since $\|d\|_2=1$, $\|d_j\|_2=1$, and $\sum_jw_j=1$, minimizing the objective is equivalent to maximizing
\[
\left\langle
d,
\sum_j w_jd_j
\right\rangle.
\]
By Cauchy--Schwarz, the maximizer is the normalized weighted mean direction.

For the gate, consider the nonnegative least-squares projection
\[
\alpha_t
=
\arg\min_{\alpha\in[0,1]}
\|\hat h_t-\alpha \hat h_q\|_2^2.
\]
Since $\|\hat h_q\|_2=\|\hat h_t\|_2=1$, the unconstrained minimizer is
\[
\alpha_t^\star
=
\langle \hat h_t,\hat h_q\rangle.
\]
Projecting it to $[0,1]$ gives
\[
\alpha_t
=
\max(0,\langle \hat h_t,\hat h_q\rangle)
=
\max(0,\cos(\hat h_t,\hat h_q)).
\]
Thus, the gate is the least-squares coefficient measuring how much the current decoding state lies in the local query direction.
\end{proof}

Combining the least-squares prototype estimator $\bar d_q$ with the least-squares local activation coefficient $\alpha_t$, the lowest-cost rank-one hidden adapter takes the form
\[
h_t'
=
h_t^{\mathrm{out}}
+
\eta\alpha_t\bar d_q.
\]
This update realizes the retrieved historical correction as a single vector addition in hidden space, modulated by local state similarity, and is discarded after generation. 
Therefore, the temporary hidden adapter can be interpreted as the local least-squares projection of multiple retrieved historical correction signals onto a single low-cost inference-time correction field.

\begin{theorem}[Direction coverage bound for old-stage risk]
\label{thm:directional_coverage}
Consider an old-stage example $z$ evaluated by the final frozen backbone. Let
\[
\mathcal L_z^\alpha(\eta,d)
=
\frac{1}{|\mathcal Y_z|}
\sum_{t\in\mathcal Y_z}
\ell_{y_t}
\left(
h_t^{\mathrm{out}}+\eta \alpha_t d
\right),
\]
where $\alpha_t\in[0,1]$ is the inference-time activation gate. Define the gated old-task correction vector
\[
G_z^\alpha
=
\frac{1}{|\mathcal Y_z|}
\sum_{t\in\mathcal Y_z}
\alpha_t W_{\mathrm{LM}}^\top(e_{y_t}-p_t).
\]
Let $b_z=G_z^\alpha/\|G_z^\alpha\|_2$ when $G_z^\alpha\neq 0$. 
Suppose the memory retrieval and aggregation procedure returns a unit direction $\bar d_z$. 
Define the directional coverage error
\[
\varepsilon_{\mathcal B}(z)
=
1-\langle b_z,\bar d_z\rangle.
\]
If the token loss is $\beta$-smooth in hidden space, then
\[
\mathcal L_z^\alpha(\eta,\bar d_z)
\le
\mathcal L_z^\alpha(0,\bar d_z)
-
\eta \|G_z^\alpha\|_2
\bigl(1-\varepsilon_{\mathcal B}(z)\bigr)
+
\frac{\beta\eta^2}{2}.
\]
\end{theorem}

\begin{proof}
For fixed $z$ and $\bar d_z$, define
\[
\phi_z(\eta)=\mathcal L_z^\alpha(\eta,\bar d_z).
\]
Its derivative at $\eta=0$ is
\[
\phi_z'(0)
=
\frac{1}{|\mathcal Y_z|}
\sum_{t\in\mathcal Y_z}
\alpha_t
\left\langle
\nabla_h \ell_{y_t}(h_t^{\mathrm{out}}),
\bar d_z
\right\rangle.
\]
Using
\[
\nabla_h \ell_{y_t}(h_t^{\mathrm{out}})
=
W_{\mathrm{LM}}^\top(p_t-e_{y_t}),
\]
we get
\[
\phi_z'(0)
=
-\left\langle
G_z^\alpha,\bar d_z
\right\rangle
=
-\|G_z^\alpha\|_2
\langle b_z,\bar d_z\rangle.
\]
Since $\langle b_z,\bar d_z\rangle=1-\varepsilon_{\mathcal B}(z)$, and by $\beta$-smoothness,
\[
\phi_z(\eta)
\le
\phi_z(0)+\eta\phi_z'(0)+\frac{\beta\eta^2}{2},
\]
which gives the desired result.
\end{proof}

Let $\mathcal P_{\mathrm{old}}$ be the distribution over previously seen examples. Denote the frozen final old-task risk by
\[
\mathcal R_{\mathrm{old}}^{0}
=
\mathbb E_{z\sim\mathcal P_{\mathrm{old}}}
[
\mathcal L_z^\alpha(0,\bar d_z)
],
\]
and the adapted risk by
\[
\mathcal R_{\mathrm{old}}^{\mathcal B}(\eta)
=
\mathbb E_{z\sim\mathcal P_{\mathrm{old}}}
[
\mathcal L_z^\alpha(\eta,\bar d_z)
].
\]
Then
\[
\mathcal R_{\mathrm{old}}^{\mathcal B}(\eta)
\le
\mathcal R_{\mathrm{old}}^{0}
-
\eta
\mathbb E_z
\left[
\|G_z^\alpha\|_2
\langle b_z,\bar d_z\rangle
\right]
+
\frac{\beta\eta^2}{2}.
\]
Thus, the benefit of memory is controlled by the directional coverage term
\[
\mathrm{Cov}(\mathcal B)
=
\mathbb E_z
\left[
\|G_z^\alpha\|_2
\langle b_z,\bar d_z\rangle
\right].
\]
If memory preserves old-stage directions with positive alignment, then inference-time adaptation reduces old-stage risk to first order. Relative to any reference old-stage risk $\mathcal R_{\mathrm{old}}^{\mathrm{ref}}$, the forgetting gap satisfies
\[
\mathcal R_{\mathrm{old}}^{\mathcal B}(\eta)
-
\mathcal R_{\mathrm{old}}^{\mathrm{ref}}
\le
\mathcal R_{\mathrm{old}}^{0}
-
\mathcal R_{\mathrm{old}}^{\mathrm{ref}}
-
\eta \mathrm{Cov}(\mathcal B)
+
\frac{\beta\eta^2}{2}.
\]

\section{Extended Methodology}
\label{app:method}

\subsection{Inference-Time Efficiency}
\label{app:inference_efficiency}

At inference time, \sysname{} adds only retrieval, reranking, and temporary hidden-state adaptation on top of standard generation. 
Exact semantic and directional recall cost $O(AH)$ and $O(JH)$, respectively, where $A=\sum_j|\mathcal A_j|$ is the number of anchors and $J$ is the number of prototypes. 
Both are bounded by the fixed memory budget and can be further accelerated with approximate indexing. 
The uncertainty proxy and reranking cost $O(K_vH)$ and $O(|\mathcal C_q|K_vH)$, respectively, where $K_v$ is the restricted vocabulary size and $\mathcal C_q$ is the candidate set. 
With fixed retrieval hyperparameters and bounded memory, the additional query-level overhead is linear in the hidden dimension $H$. 
During generation, the adapter introduces only a cosine gate and a vector addition per token, which also cost $O(H)$. 
Therefore, \sysname{} avoids prompt extension, replay, parameter updates, and test-time backpropagation, making it compatible with the low-storage and low-latency constraints of SLM deployment.

\subsection{Detailed Implementation of \sysname{}}
\label{app:miita_algorithm}

Algorithm~\ref{alg:miita} summarizes the full implementation of \sysname{}, including memory construction, budget-aware memory repair, inference-time retrieval, and temporary hidden-state adaptation.

\begin{algorithm}[t]
\caption{\sysname{}: Memory-Induced Inference-Time
Adaptation Framework for CL}
\label{alg:miita}
\begin{algorithmic}[1]
\Require Stream $\{\mathcal D_t\}_{t=1}^{T}$, base SLM $f$, budget $C_{\mathrm{mem}}$, hyperparameters $\{\tau_d,\tau_k,R_k,R_d,M,\lambda_k,\lambda_u,T_{\mathrm{agg}},\eta\}$
\Ensure Memory bank $\mathcal B$, prediction $\hat y$

\State $\mathcal B\leftarrow\varnothing$

\Statex
\Statex \textit{// Phase 1: Build budgeted direction memory}
\For{$t=1$ to $T$}
    \For{each $(x_m,y_m)\in\mathcal D_t$}
        \State $k_m \leftarrow$ Eq.~\eqref{eq:semantic_key}, \quad $d_m \leftarrow$ Eq.~\eqref{eq:correction_direction}
        \If{$\mathcal B=\varnothing$}
            \State Add $u=(d_m,\{k_m\},1)$ to $\mathcal B$
        \Else
            \State $j^\star \leftarrow \arg\max_j \cos(d_m,d_j)$
            \If{$\cos(d_m,d_{j^\star})<\tau_d$}
                \State Add $u=(d_m,\{k_m\},1)$ to $\mathcal B$
            \Else
                \State Update $u_{j^\star}$ by Eqs.~\eqref{eq:prototype_update} and~\eqref{eq:anchor_update}
            \EndIf
        \EndIf
        \While{$\mathrm{Size}(\mathcal B)>C_{\mathrm{mem}}$}
            \State Repair $\mathcal B$ by anchor pruning, prototype merging, or eviction via Eq.~\eqref{eq:prototype_value}
        \EndWhile
    \EndFor
\EndFor

\Statex
\Statex \textit{// Phase 2: Retrieve and aggregate memory directions}
\State Given query $q$, compute $k_q$ by Eq.~\eqref{eq:semantic_key} and $\tilde d_q$ by Eq.~\eqref{eq:uncertainty_direction}
\State Construct $\mathcal C_q$ by Eq.~\eqref{eq:candidate_set}
\For{each $j\in\mathcal C_q$}
    \State Compute $\mathrm{score}(q,j)$ by Eq.~\eqref{eq:retrieval_score}
\EndFor
\State $S_q \leftarrow \operatorname{Top}_{M}\{\mathrm{score}(q,j):j\in\mathcal C_q\}$
\State Aggregate $\bar d_q$ by Eq.~\eqref{eq:direction_aggregation}

\Statex
\Statex \textit{// Phase 3: Apply temporary hidden adaptation}
\For{each decoding step $s$}
    \State Apply gated update $h_s'$ by Eq.~\eqref{eq:hidden_adaptation}
    \State Generate next token from $h_s'$
\EndFor
\State Discard the temporary update
\State \Return $\hat y$
\end{algorithmic}
\end{algorithm}

\section{Experimental Details}
\subsection{Baseline Details}
\label{app:baselines}

\textbf{General protocol.}
All baselines are evaluated under the same budget-constrained supervised SLM CL setting. 
Unless otherwise specified, each benchmark is converted into the unified instruction-style format $(x,y)$, and all methods use the same SLM backbone. 
For methods that store auxiliary information, including replay samples, summaries, statistics, episodic memories, adapters, or bottleneck parameters, we constrain the total persistent storage to the same budget $C_{\mathrm{bytes}}$. 
Methods designed only for classification are evaluated on the classification benchmarks and are marked as not applicable to open-ended QA.

\textbf{Baseline descriptions and adaptations.}
\begin{itemize}[leftmargin=*, itemsep=0pt, topsep=0pt]
    \item \textbf{LoRA}~\cite{LoRA}. 
    LoRA is a parameter-efficient sequential tuning baseline that updates low-rank adapters while keeping the SLM backbone frozen. 
    We train one LoRA module sequentially on each incoming stage using only the current data $\mathcal{D}_t$, without replay or external memory. 
    The LoRA parameters are carried across stages as the model's continual update state.

    \item \textbf{MBPA++}~\cite{MBPA}. 
    MBPA++ is an episodic-memory language-model CL method that stores past examples and uses sparse replay together with inference-time local adaptation. 
    We replace its original language model with the same SLM backbone used by our method, store serialized $(x,y)$ examples in episodic memory under $C_{\mathrm{bytes}}$, retrieve nearest neighbors using frozen SLM representations, and perform temporary local LoRA adaptation for each query before discarding the update.

    \item \textbf{CT0}~\cite{CT0}. 
    CT0 is a training-time rehearsal baseline that continually learns new instruction-following tasks by replaying a small buffer of previous examples. 
    We adapt it by maintaining a budgeted replay buffer of old supervised examples and mixing sampled replay data with the current stage data during LoRA-based continual tuning.

    \item \textbf{FOREVER}~\cite{FOREVER}. 
    FOREVER is a replay-based LLM CL method that schedules memory replay according to model-centric update dynamics. 
    We use the same budgeted replay buffer as CT0, but trigger replay according to FOREVER's model-time schedule computed from accumulated LoRA update norms, and apply its replay-time regularization during continual tuning.

    \item \textbf{CIS}~\cite{CIS}. 
    CIS is the concrete method under the CLOB paradigm, which performs black-box CL by storing class knowledge as incrementally updated textual summaries. 
    For classification benchmarks, we follow its original design by generating and updating class-level summaries under $C_{\mathrm{bytes}}$ and using them in inference-time prompts. 
    For QA benchmarks, where the output space is open-ended rather than class-based, we adapt CIS by maintaining stage-level or domain-level summaries of previously seen QA examples. 
    At inference time, the query is paired with the relevant stored summaries in the prompt, and the SLM generates the answer directly.

    \item \textbf{InCA}~\cite{InCA}. 
    InCA is an inference-time ICL method that uses an external continual learner to select candidate classes before prompting the LLM. 
    For classification benchmarks, we keep this mechanism and incrementally store class-level tag statistics and summaries under $C_{\mathrm{bytes}}$. 
    For QA benchmarks, we adapt the external learner from class selection to memory selection: the learner ranks stored stage/domain summaries or representative QA memories according to the semantic tags of the query, and the top-ranked memories are used to construct the inference prompt. 
    The final prediction is generated by the SLM rather than selected from a finite class set.

    \item \textbf{DKVB}~\cite{DKVB}. 
    DKVB is an SLM-oriented CL method that uses a discrete key-value bottleneck to enable localized continual updates. 
    Since the original method is designed for encoder-only classification, we adapt it by inserting the bottleneck at the SLM hidden representation before the prediction head. 
    For instruction-style tasks, labels are verbalized as output tokens, and the bottleneck capacity is selected to fit the same storage budget.

    \item \textbf{Srinath K \textit{et al.}}~\cite{POS}. 
    Srinath K \textit{et al.} propose an SLM-oriented multilingual CL method for mitigating cross-lingual forgetting with replay adapters. 
    For multilingual QA, we treat each language as a domain and use its replay-adapter strategy in the original language-incremental form. 
    For non-multilingual benchmarks, we adapt the method by retaining the shared replay-adapter mechanism but replacing POS-guided code-switching with task- or domain-preserving replay examples from the same benchmark. 
    Thus, the baseline becomes a replay-adapter CL method that uses the same instruction-style input-output format while removing the language-specific code-switching component when it is not applicable.
\end{itemize}

\subsection{Details of Memory Item Variants}
\label{app:memory_item_variants}

We evaluate four memory item formats under the same retrieval backbone, retrieval budget, and total storage constraint. 

\begin{itemize}[leftmargin=*, itemsep=0pt, topsep=0pt]
    \item \textbf{ExampleMem.}
    This variant stores serialized supervised examples $(x,y)$ as memory items. 
    At inference time, retrieved examples are inserted into the prompt as demonstrations or contextual evidence. 
    The storage cost is measured by the byte size of the stored text.

    \item \textbf{SummaryMem.}
    This variant compresses historical examples into natural-language summaries. 
    Retrieved summaries are inserted into the prompt before generation. 
    The storage cost includes the byte size of the stored summaries.

    \item \textbf{SemanticMem.}
    This variant stores semantic embedding prototypes obtained from the same retrieval backbone. 
    At inference time, retrieved semantic prototypes are used through the same retrieval interface, but they do not contain supervised correction information. 
    The storage cost includes the stored embedding vectors and associated prototype metadata.

    \item \textbf{DirectionMem.}
    This is the memory format used by \sysname{}. 
    Each supervised example induces a correction direction in hidden space, and similar directions are merged into prototypes under the memory budget. 
    At inference time, retrieved direction prototypes are used for temporary hidden-state adaptation without gradient-based test-time training.
\end{itemize}

\subsection{Details of Retrieval Variants}
\label{app:retrieval_variants}

We evaluate retrieval variants under the same DirectionMem memory bank, retrieval backbone, retrieval budget, and total storage constraint. 
All variants use the same supervised stream and Qwen3-0.6B backbone. 
They differ only in how memory prototypes are selected or reranked at inference time.

\begin{itemize}[leftmargin=*, itemsep=0pt, topsep=0pt]
    \item \textbf{RandomRet.}
    This variant uniformly samples memory prototypes from the DirectionMem bank. 
    It ignores both semantic relevance and uncertainty information and serves as a lower-bound retrieval baseline.

    \item \textbf{SemanticRet.}
    This variant retrieves memory prototypes using only semantic-key similarity $S_k(q,j)$ between the query and stored memory anchors. 
    It tests whether semantic relevance alone is sufficient for selecting useful historical correction directions.

    \item \textbf{UncertaintyRet.}
    This variant retrieves memory prototypes using only the uncertainty-induced query direction $\tilde d_q$. 
    Since $\tilde d_q$ and the stored correction directions $d_j$ are both defined in the LM-head input hidden space, they can be compared by directional alignment. 
    This variant tests whether uncertainty provides a label-free query-side signal for selecting historical directions that may locally affect the current prediction distribution.

    \item \textbf{SemUncRet.}
    This is the retrieval strategy used by \sysname{}. 
    It first applies semantic recall to construct a query-relevant candidate memory set and then applies uncertainty-guided soft reranking within this set. 
    Semantic recall keeps retrieved memories related to the query, while uncertainty prioritizes directions that are more likely to influence the current predictive state.

    \item \textbf{OracleRet.}
    This variant uses the ground-truth target to compute an oracle correction direction and retrieves memory prototypes according to their alignment with this direction. 
    For a single target token, the oracle direction is computed as $d_q^\star=W_{\mathrm{LM}}^\top(e_y-p)$, where $p$ is the predicted distribution and $e_y$ is the one-hot target distribution. 
    For generation tasks, token-level oracle directions are averaged over the target sequence. 
    OracleRet is used only as an upper bound and is not available at inference time.
\end{itemize}

\paragraph{Oracle Alignment.}
For retrieval analysis, we compute \textbf{Oracle Alignment (OA)} on a held-out evaluation set where labels are available. 
For each query $q$, we first compute the oracle correction direction $d_q^\star$ from the ground-truth target. 
Each retrieval variant then produces an aggregated retrieved direction $\bar d_q$ using the same direction aggregation rule as \sysname{}. 
OA is defined as the average cosine similarity between $\bar d_q$ and $d_q^\star$:
\[
\mathrm{OA}
=
\frac{1}{|\mathcal Q|}
\sum_{q\in\mathcal Q}
\cos(\bar d_q,d_q^\star).
\]
A higher OA means that the retrieved direction has stronger alignment with the label-derived oracle correction direction.
OA should not be interpreted as accuracy or as the fraction of queries whose predictions become correct. 
It measures directional agreement in hidden space. 
In high-dimensional spaces, random directions have cosine similarity close to zero, so an OA around $0.5$ already indicates a strong positive projection onto the oracle correction direction. 
A retrieved direction does not need to be perfectly collinear with the oracle direction to be useful: as long as it has a positive projection on the oracle direction, a small hidden-state update can still increase the probability of target tokens or reduce the supervised loss. 
Conversely, OA close to $1$ would mean near-perfect recovery of the label-derived correction direction, which is unrealistic for a label-free retrieval signal in generation tasks. 
Even OracleRet does not obtain OA equal to $1$ because it still retrieves and aggregates finite historical memory prototypes rather than directly applying the oracle direction itself. 
Thus, OA is used only as a diagnostic measure of retrieval-direction quality, while OP and BWT measure the final continual learning performance.

\section{Supplementary Experiments}
\label{app:supplementary_experiments}

\subsection{Effect of the Activation Gate}
\label{app:gate_ablation}

We additionally evaluate the activation gate $\alpha_t$ used in the temporary hidden adaptation step. 
The experiment is conducted on \textsc{SuperNI} with Qwen3-0.6B under the fixed 1\% memory budget. 
The full \sysname{} uses the gated update
$h_t'=h_t^{\mathrm{out}}+\eta\alpha_t\bar d_q$, while the ablated variant removes the gate by setting $\alpha_t=1$ for all decoding steps, so that the same correction direction is uniformly applied throughout generation.

Removing the gate reduces performance from $40.58$ OP / $-4.89$ BWT to $38.72$ OP / $-6.18$ BWT. 
This confirms that the gate is not a redundant component: uniformly injecting the retrieved direction can over-correct decoding states that are no longer close to the query state that triggered retrieval. 
In contrast, the full gated design applies stronger correction only when the current decoding state remains locally aligned with the query state. 
This is consistent with Proposition~\ref{prop:hidden_adapter_lsq}, which interprets $\alpha_t$ as a nonnegative least-squares projection coefficient measuring how much the current decoding state lies in the local query direction.

\subsection{Hyperparameter Sensitivity}
\label{app:hyperparameter_sensitivity}

We further study the sensitivity of \sysname{} to the two memory-construction thresholds, $\tau_d$ and $\tau_k$, on \textsc{SuperNI} with Qwen3-0.6B under the fixed 1\% memory budget. 
When sweeping one hyperparameter, all others are fixed to the default configuration. 
The default setting used in the main experiments is $\tau_d=0.85$ and $\tau_k=0.80$.

\begin{table}[t]
\centering
\small
\setlength{\tabcolsep}{6pt}
\caption{
Hyperparameter sensitivity on \textsc{SuperNI} with Qwen3-0.6B under a fixed 1\% memory budget. 
When sweeping one threshold, all other hyperparameters are fixed to the default configuration. 
\textbf{Bold} indicates the default setting used in the main experiments.
}
\label{tab:hyperparameter_sensitivity}
\begin{tabular}{lccc}
\toprule
\textbf{Hyperparameter} & \textbf{Value} & OP$\uparrow$ & BWT$\uparrow$ \\
\midrule
\multirow{5}{*}{$\tau_d$}
& 0.75 & 39.82 & -5.31 \\
& 0.80 & 40.21 & -5.08 \\
& \textbf{0.85} & \textbf{40.58} & \textbf{-4.89} \\
& 0.90 & 40.34 & -5.02 \\
& 0.95 & 39.76 & -5.36 \\
\midrule
\multirow{5}{*}{$\tau_k$}
& 0.70 & 40.05 & -5.17 \\
& 0.75 & 40.31 & -5.02 \\
& \textbf{0.80} & \textbf{40.58} & \textbf{-4.89} \\
& 0.85 & 40.39 & -4.97 \\
& 0.90 & 40.12 & -5.15 \\
\bottomrule
\end{tabular}
\end{table}

Table~\ref{tab:hyperparameter_sensitivity} shows that \sysname{} is robust across a reasonable range of threshold choices. 
For $\tau_d$, overly small values merge directions too aggressively and may blur distinct functional corrections, while overly large values create more fragmented prototypes and reduce the benefit of direction sharing. 
The default value $\tau_d=0.85$ provides the best balance between compression and functional diversity. 
For $\tau_k$, lower values keep too few semantic anchors and may weaken retrieval access, while higher values preserve more anchors but consume budget with redundant semantic entry points. 
The default value $\tau_k=0.80$ achieves the best overall trade-off. 
Across all tested values, performance remains close to the default setting, indicating that the gains of \sysname{} are not due to a fragile hyperparameter choice.

\section{Limitations}
\label{app:limitations}

MIITA is designed for supervised continual learning with small language models under a fixed persistent-memory budget, and its conclusions should be interpreted within this setting. The method assumes that each incoming stage provides input-output supervision from which correction directions can be extracted, and it does not directly cover unsupervised, reinforcement-learning, multimodal, or tool-use continual adaptation. Because the backbone remains frozen and adaptation is applied only as a temporary hidden-state perturbation, MIITA cannot recover capabilities that are absent from the base model or fully replace parameter updates when new tasks require substantial new reasoning skills or knowledge that is poorly represented in the existing hidden space.

The effectiveness of MIITA also depends on the quality and coverage of the stored direction prototypes. If historical examples induce highly diverse or conflicting correction directions under a very small memory budget, memory repair may merge or evict directions that are useful for rare tasks. Similarly, retrieval relies on semantic anchors and an uncertainty-based label-free proxy, so out-of-domain queries, ambiguous prompts, poorly calibrated predictive uncertainty, or weak semantic representations may retrieve suboptimal directions and reduce the benefit of inference-time adaptation. The theoretical analysis is local: it relies on differentiability, first-order approximations around the current hidden state, fixed LM-head geometry, and smoothness assumptions, and therefore does not guarantee global performance improvements for large adaptation scales or across arbitrary distribution shifts.

Although MIITA avoids prompt extension, test-time backpropagation, and persistent backbone updates, it still introduces additional inference-time computation for memory retrieval, reranking, and per-token hidden-state gating. Exact retrieval scales with the number of stored anchors and prototypes, so very long streams or larger memory budgets may require approximate indexing or further compression to maintain low latency. Our experiments cover multiple supervised CL benchmarks, task types, model families, model scales, and memory budgets, but they do not exhaustively evaluate very long-horizon streams, safety-critical deployments, non-English domains beyond the tested multilingual QA setting, or all possible SLM architectures. Finally, storing compact correction directions reduces the amount of raw text retained but should not be interpreted as a formal privacy guarantee; if the training stream contains sensitive or biased supervision, additional auditing, privacy protection, and fairness evaluation would be needed before deployment.

\end{document}